%% file: paper.tex
\documentclass[runningheads]{llncs}
\input{general/packages}
\input{general/macros}

\begin{document}
\title{Unsupervised Automata Learning via Discrete Optimization}
%
%
\author{
Simon Lutz\inst{1,4}\and
Daniil Kaminskyi\inst{1,4} \and
Florian Wittbold\inst{3}\and
Simon Dierl\inst{1}\and
Falk Howar\inst{1,2}\and
Barbara König\inst{3}\and
Emmanuel Müller\inst{1}\and
Daniel Neider\inst{1,4}
}

\institute{
    TU Dortmund University, Germany \\
    \email{\{simon.lutz, daniil.kaminskyi, simon.dierl, falk.howar, daniel.neider\}@tu-dortmund.de\\
           emmanuel.mueller@cs.tu-dortmund.de}
\and 
    Fraunhofer ISST, Germany
\and
    University of Duisburg-Essen, Germany \\  
    \email{\{barbara\_koenig, florian.wittbold\}@uni-due.de}
\and
    Center for Trustworthy Data Science and Security, UA Ruhr, Germany
}
\authorrunning{Lutz et al.}
%
%
\maketitle              
\begin{abstract}
Automata learning is a successful tool for many application domains such as robotics and automatic verification.
Typically, automata learning techniques operate in a supervised learning setting (active or passive) where they learn a finite state machine in contexts where additional information, such as labeled system executions, is available.
However, other settings, such as learning from unlabeled data - an important aspect in machine learning - remain unexplored.
To overcome this limitation, we propose a framework for learning a deterministic finite automaton (DFA) from a given multi-set of unlabeled words.
We show that this problem is computationally hard and develop three learning algorithms based on constraint optimization.
Moreover, we introduce novel regularization schemes for our optimization problems that improve the overall interpretability of our DFAs.
Using a prototype implementation, we demonstrate practical feasibility in the context of unsupervised anomaly detection.

\end{abstract}
%
%
%
\input{intro}

\input{prelim}

\input{problem}

\input{learning}

\input{interpretability}

\input{evaluation}

\input{conclusion}

\chapter*{Acknowledgements}
\label{sec:Ack}
This work has been financially supported by Deutsche Forschungsgemeinschaft, DFG Project numbers 459419731, 495857894
(STING), and 434592664, and the Research Center Trustworthy Data Science and Security (https://rc-trust.ai), one of the Research Alliance centers within the UA Ruhr (https://uaruhr.de).

\bibliographystyle{splncs04}
\bibliography{paper}

\end{document}

%% file: general/packages.tex
\usepackage[T1]{fontenc}
\usepackage{times}
\usepackage{soul}
\usepackage{url}
\usepackage[hidelinks]{hyperref} 
\usepackage{color}
\usepackage[utf8]{inputenc}
\usepackage[small]{caption}
\usepackage{graphicx}
\usepackage{amsmath}
\usepackage{booktabs}
\usepackage{algorithm}
\usepackage{algpseudocode}
\usepackage{apxproof}
\usepackage[switch]{lineno}
\usepackage{microtype}
\usepackage{amssymb}
\usepackage{ marvosym }
\usepackage{fontawesome5}
\usepackage{mathtools}
\usepackage{pgfplots}
\pgfplotsset{compat=1.16}
\usepackage{tikz}
\usetikzlibrary{backgrounds,calc,positioning}
\usetikzlibrary{overlay-beamer-styles}
\usetikzlibrary{arrows,shapes,snakes,automata}

\usepackage{xspace}
\usepackage{graphicx}
\usepackage{ latexsym }

\usepackage[page]{appendix}

\usepackage{subcaption}
\usepgfplotslibrary{colorbrewer}

\theoremstyle{definition}

%% file: general/macros.tex
\newcommand{\dfa}{\mathcal{A}}
\newcommand{\sample}{\mathcal{S}}
\newcommand{\abs}[1]{\ensuremath{|#1|}}
\newcommand{\NP}{\ensuremath{\mathsf{NP}}\xspace}
\newcommand{\NPO}{\ensuremath{\mathsf{NPO}}\xspace}

\newtheoremrep{example}{Example}
\newtheoremrep{theorem}{Theorem}
\newtheoremrep{problem}{Problem}
\newtheoremrep{lemma}{Lemma}

%% file: intro.tex
\section{Introduction}
\label{sec:intro}

In the last decades, the algorithmic learning of finite automata (or automata learning for short) has proven to be a successful tool in many application domains, ranging from pattern and  language recognition~\cite{garcia1990use} over robotics~\cite{rivest1989inference,rieger1995inferring} to automatic verification~\cite{oliveira2001efficient,groce2006adaptive,habermehl2005regular} and software testing~\cite{aichernig2018model}.
For reactive system verification, for instance, the goal of automata learning is to provide an appropriate abstraction of the system's input–output relations as a finite-state machine~\cite{hungar2003domain}.

Traditionally, the literature on automata learning distinguishes two main settings: active and passive learning.
In active learning~\cite{Angluin87}, the algorithm (called learner) interacts with a so-called teacher.
This teacher has access to a regular language and is able to answer two types of queries.
Membership queries ask whether a specific word is in the target language and equivalence queries ask whether a conjectured automaton
is equivalent to the language in question.
In passive learning~\cite{BiermannF72,Trakhtenbrot1973FiniteA}, the learning algorithm is given a finite set of words which are labeled as positive or negative, i.e., whether they are contained in the regular target language or not, respectively.
Then, the objective is typically to learn a minimal automaton that accepts all positive words and rejects all negative ones. 

While many advances in active and passive learning have expanded upon these seminal works, other important learning settings remain unexplored.
For instance, the field of unsupervised learning is a well-studied aspect in machine learning that, so far, has been ignored in the context of automata learning.
However, many important unsupervised learning problems, such as anomaly detection, also arise for automata and reactive systems.
Currently, they are addressed via use case specific solutions, which are hard to engineer and difficult to transfer to other settings.

To overcome this gap, this paper proposes a generic approach for unsupervised automata learning based on discrete optimization.
Similar to passive learning, we rely on a given, finite set of words but assume that, a priori, no additional information, such as positive or negative labels, is available.
While our ideas are applicable to many other unsupervised learning settings on sequential data, in this paper, we focus on a crucial application in unsupervised machine learning: anomaly detection (i.e., identifying patterns in the data that do not conform to expected behavior~\cite{DBLP:journals/csur/ChandolaBK09}).
This choice is motivated by the many application domains of anomaly detection, including cybersecurity, law enforcement, medicine, and fraud detection, to name but a few.


To be more precise, we aim to learn a DFA from a given finite multi-set $\sample$ of unlabeled sequences that can distinguish normal from anomalous sequences. To this end, we consider three unsupervised learning settings.
In the first setting, we assume two natural numbers $\ell, u \in \mathbb N$ with $\ell \leq u\leq\abs{\sample}$ to be given as input.
The task is then to learn a minimal DFA that accepts at least $\ell$ and at most $u$ sequences from~$\sample$.
Minimality refers to a minimal number of states and is a common requirement in automata learning~\cite{BiermannF72,DBLP:conf/icgi/HeuleV10,DBLP:conf/isola/LeuckerN12,DBLP:conf/aaai/NeiderGGT0021}.
The parameters $\ell$ and $u$, on the other hand, serve as an estimate for the lower and upper number of anomalies in the data set and are used to prevent degenerate DFAs (i.e., DFAs that accept or reject all sequences).
This setting operates under the assumption that normal sequences are drastically different from anomalies, allowing them to be separated by a rather simple pattern.
Hence, by looking for an automaton that is as compact as possible, the
classification of anomalies is performed automatically.

In the second setting the user does not need to specify both $\ell$ and $u$, but only one or the other (say, $\ell$).
Additionally, the user must fix a size $n \in \mathbb N$ of the resulting DFA.
The task then is to learn a DFA of size $n$ that accepts the smallest number $k \geq \ell$ of sequences from $\sample$.
In other words, $\ell$ serves as a lower bound on the assumed number of anomalies in the given data set.
In general, the choice of $n$ should be made carefully in this setting, as too large a number may hinder interpretability while, if $n$ is too small, the resulting DFAs may not be able to separate anomalies from normal sequences.

The last setting is motivated by the assumption that all normal sequences are similar to each other, i.e., have a small edit distance, while the anomalies are vastly different from the normal sequences, resulting in a high edit distance.
Under this assumption the user does not need to specify any bounds, but must still fix a size $n \in \mathbb N$ of the resulting DFA.
We learn a DFA that minimizes the distance between pairs of sequences classified as normal, while maximizing the distance between pairs of sequences where one is classified as normal and the other as an anomaly.

Our contributions in this paper are fourfold.
First, we show that learning a DFA of size $n$ from unlabeled data is \NP-complete. This result is in line with the classical learning of DFAs from positive and negative data, which is known to be \NP-complete as well~\cite{DBLP:journals/iandc/Gold78}. Consequently, the first learning problem lies within the complexity class FNP (i.e., the function problem extension of the decision problems in NP) and the second one lies within the class \NPO (i.e., the class of optimization problems whose decision variant lies in \NP).
The complexity of the third learning problem remains an open problem that we will leave as a part of future work.

Second, we develop three learning algorithms, one for each setting.
While in previous work, a DFA was learned by solving a series of constraint satisfaction problems, we reduce learning into a series of constraint optimization problems instead.
This allows us to specify an objective function, thus finding not just any solution but one that is optimized for additional regularization criteria.
The constraint optimization problems can then be solved by highly-optimized mixed-integer programming solvers.

Third, we propose novel regularization terms to enhance the interpretability of the learned DFAs.
In particular, we show how to augment our constraint optimization problems to maximize the number of self-loops and parallel edges (see Figure~\ref{fig:interpretability}).
This approach is orthogonal to the original encoding and can, in principle, also be applied to other constraint-based learning algorithms for finite-state machines.

Fourth, to show the practical feasibility of our three algorithms, we evaluate them empirically on three anomaly detection benchmarks.
We examine both the runtime and the anomaly detection performance for different configuration options and uncertainty w.r.t.~the anomaly frequency.

\subsubsection*{Related Work}
Automata learning has a long history, dating back to the 1970s~\cite{BiermannF72,Trakhtenbrot1973FiniteA}.
One typically distinguishes between active learning and passive learning.

Active learning was first introduced by Dana Angluin in 1987~\cite{Angluin87}.
In her work, Angluin showed that the class of regular languages can be learned efficiently by asking queries to a (minimally adequate) teacher.
Furthermore, she provided an appropriate learning algorithm - called the $\text{L}^*$ algorithm - which approximates the Myhill-Nerode congruence.
Since then, various major improvements to and variants of the original algorithm have been proposed~\cite{RivestS93,Kearns94,MalerP95,Irfan10,MertenHSM11,Howar12,aarts2014tomte,IsbernerS14,IsbernerHS14,Petrenko0GHO14,Frohme19,VaandragerGRW22}.

Passive learning, on the other hand, was pioneered by Biermann and Feldman~\cite{BiermannF72} and Trakhtenbrot and Barzdin~\cite{Trakhtenbrot1973FiniteA}.
Given a set of labeled data, a passive learning algorithm seeks to learn a minimal DFA consistent with the data. 
Algorithms such as Regular Positive Negative Inference (RPNI)~\cite{oncina1992inferring} and the Blue-fringe algorithm~\cite{DBLP:conf/icgi/LangPP98} first construct the prefix acceptor -- the most precise description of the data -- and then generalize it by merging its states while maintaining consistency with the data.
In 1987, Gold~\cite{DBLP:journals/iandc/Gold78} showed that passive learning is computationally hard (i.e., the corresponding decision problem is \NP-complete).
Thus, learning algorithms that use constraint solving have become the de~facto standard for constraint-based passive learning~\cite{DBLP:conf/cade/GrinchteinLP06,DBLP:conf/icgi/HeuleV10,DBLP:conf/atva/Neider12,DBLP:conf/nfm/NeiderJ13,DBLP:conf/aaai/NeiderGGT0021}.


Besides the traditional active and passive learning of deterministic finite automata, other learning settings have also been investigated.
Following the same underlying ideas, various algorithms have been proposed in the literature for learning more expressive state machines such as 
Mealy Machines with and without timers~\cite{Niese2003,ShahbazG09,DBLP:journals/corr/abs-2403-02019},
I/O automata~\cite{DBLP:conf/concur/AartsV10},
non-deterministic automata~\cite{BolligHKL09,BjorklundFK13}, 
alternating automata\cite{DBLP:conf/ijcai/AngluinEF15},
register automata~\cite{DBLP:conf/fm/AartsHKOV12,DBLP:journals/fac/CasselHJS16,DBLP:conf/ifm/GarhewalVHSLS20,DBLP:journals/ml/IsbernerHS14},
weighted automata~\cite{BergadanoV96,BalleM15,HeerdtKR020},
pushdown automata~\cite{DBLP:phd/dnb/Isberner15},
tree automata~\cite{DBLP:journals/tcs/KnuutilaS94,oncina1993inference}, among others.
In the context of incomplete information, Leucker and Neider~\cite{DBLP:conf/isola/LeuckerN12} proposed a variant of Angluin's $\text{L}^*$ algorithm for learning from an inexperienced teacher that sometimes may answer ``don't know'' to a membership query.
However, to the best of our knowledge, learning a Deterministic Finite Automaton completely from unlabeled data remains unexplored.

%% file: prelim.tex
\section{Preliminaries}
We address the task of learning a deterministic finite automaton from a multi-set of unlabeled sequences ranging over a finite set of symbols.


Following standard notation of automata theory, we refer to a sequence $w = a_1 \dots a_n$ as a \emph{finite word}.
Moreover, we call the nonempty, finite set of \emph{symbols} over which these words can range an \emph{alphabet $\Sigma$}.
The sequence without any symbols, also referred to as \emph{empty word}, is denoted by $\epsilon$.
Furthermore, we denote the set of all words over an alphabet $\Sigma$ as $\Sigma^*$.
In the remainder of this paper, we will refer to the multi-set of sequential data $\sample =\{w_1,\dots, w_n\}$ as a \emph{sample}. Since a word $w$ can be contained multiple times in a sample $\sample$, we denote the number of occurrences of $w$ in $\sample$ as $\sample(w)$.

From a given sample, we learn a \emph{deterministic finite automaton (DFA)}.
Formally, a DFA is a tuple $\dfa = (Q, \Sigma, q_I, \delta, F)$ where $Q$ is a finite set of states, $\Sigma$ is a finite set of (input) symbols, $q_I \in Q$ is the \emph{initial state}, $\delta: Q \times \Sigma \to Q$ is the \emph{state-transition function}, and $F \subseteq Q$ is a set of accepting states.
The \emph{size} of a DFA is defined to be the number of its states $\abs{Q}$.
A \emph{run} on a word $w = a_1 \dots a_n$ is a sequence of states $q_0 \dots q_n$ such that $q_0 = q_I$ and $q_i = \delta(q_{i-1}, a_i)$ for $i \in \{1, \dots, n\}$.
We call a run \emph{accepting} if $q_n \in F$, and \emph{rejecting} otherwise.
The \emph{language} of a DFA~$\dfa$, denoted~$L(\dfa)$, is the set of all words accepted by~$\dfa$.

As mentioned in the introduction, we reduce the task of learning a DFA into a series of \emph{mixed-integer linear programming (MILP)} problems.
Let $\mathit{Var}$ be a finite set of real variables.
An MILP problem consists of two parts, a linear function over the variables, referred to as the \emph{objective function} $\mathit{obj}$, and a conjunction of \emph{linear constraints} $\Phi$ on these variables.
The solution to such a MILP problem is an assignment $\mathit{Var} \to \mathbb{R}$, referred to as a (feasible) \emph{model}, such that the value of the objective function $\mathit{obj}$ is optimal (i.e. minimal/maximal, respectively) while satisfying $\Phi$ (i.e., all constraints).

%% file: problem.tex
\section{Problem Formulation}
\label{sec:problem}

In this section, we formally introduce our three unsupervised automata learning problems and prove that the first two are computationally hard.
As mentioned in the introduction, all three learning settings are assumed to be given a sample $\mathcal{S}$ of unlabeled words.
The first setup, which we refer to as \emph{Two-Bound DFA Learning}, additionally requires being given two natural numbers $\ell,u \in \mathbb{N}$ with $\ell \leq u \leq \abs{\sample}$.
While the precise labels of the words in the sample $\mathcal{S}$ are unknown, these numbers provide an estimate of the distribution of positive words in $\mathcal{S}$.
Then, the task is to learn a minimal DFA which accepts at least $\ell$
and at most $u$ words from $\mathcal{S}$.
We formally state this problem as:
\begin{problem}[Two-Bound DFA Learning Problem]
\label{problem:1}\\
Given a multi-set of words $\sample~=~\{w_1, \dots, w_n\}$ and two natural
numbers $\ell,u \in \mathbb{N}$ with $\ell \leq u \leq \abs{\sample}$, construct a DFA $\dfa$ which accepts at least $\ell$ and at most $u$ words
from $\sample$.
\end{problem}

Notice that one may not always find a solution for this problem.
This is illustrated using the following example.
\begin{example}
Consider the bounds $\ell = u = 1$ and the sample $\sample$ which only contains the word $'a'$ twice.
Being deterministic, every DFA has to either accept both copies of the word $'a'$ or reject them both.
Hence there does not exist a DFA fulfilling the bounds in this case. 
\end{example}

Despite this, Problem~\ref{problem:1} is decidable.
To show decidability, we first represent the sample as a \emph{prefix tree}~\cite{de2010grammatical}.
A prefix tree for a sample $\sample$ is a partial DFA (i.e., some transitions are unspecified) without final states such that, after reading a word $w \in \sample$, the DFA is in a unique state $q_w$.
An example of a prefix tree is displayed in Figure~\ref{fig:pt}.
We complete this partial DFA by adding an additional sink state that becomes the target of all unspecified transitions.
To decide whether there exists a DFA accepting at least $\ell$ and at most $u$ words, we iterate over all combinations of final states and check the number of accepted words in each case.
Since there is a unique state for each word in $\sample$, either one of these DFAs fulfills the bounds or we can conclude that none exists.
\begin{figure}
    \centering
    \begin{tikzpicture}[scale=0.5]
        \draw[thick] (0,0) circle (1);
        \node[circle, minimum size = 10mm] (0) at (0,0) {$q_0$};
        
        \draw[thick] (4,1.5) circle (1);
        \node[circle, minimum size = 10mm] (a) at (4,1.5) {$q_{a}$};
        
        \draw[thick] (4,-1.5) circle (1);
        \node[circle, minimum size = 10mm] (b) at (4,-1.5) {$q_{b}$};
        
        \draw[thick] (8,3) circle (1);
        \node[circle, minimum size = 10mm] (aa) at (8,3) {$q_{aa}$};
        
        \draw[thick] (8,0) circle (1);
        \node[circle, minimum size = 10mm] (ab) at (8,0) {$q_{ab}$};
        
        \draw[thick] (8,-3) circle (1);
        \node[circle, minimum size = 10mm] (ba) at (8,-3) {$q_{ba}$};
        
        \draw[thick] (12,3) circle (1);
        \node[circle, minimum size = 10mm] (aaa) at (12,3) {$q_{aaa}$};
        
        \draw[thick, -{stealth}] (-1.8,0) -- (0);
        \draw[thick, -{stealth}] (0) -- (a) node [midway, above] {a};
        \draw[thick, -{stealth}] (0) -- (b) node[midway,above]{b};
        \draw[thick, -{stealth}] (a) -- (aa) node[midway,above]{a};
        \draw[thick, -{stealth}] (a) -- (ab) node[midway,above]{b};
        \draw[thick, -{stealth}] (b) -- (ba) node[midway,above]{a};
        \draw[thick, -{stealth}] (aa) -- (aaa) node[midway,above]{a};
    
    \end{tikzpicture}
    \caption{Prefix tree for the sample $(aa, ab, ba, aaa)$}
    \label{fig:pt}
\end{figure}
However, a DFA constructed this way will generally be unsuitable for applications such as anomaly detection as it suffers from two main issues.
On the one hand, by design, it overfits the sample $\sample$ and thus poorly generalizes to unseen data.
On the other hand, it becomes rather large, hindering interpretability in the sense of Occam's razor.
To overcome these issues, we propose an algorithm that constructs a DFA of \emph{minimal size} fulfilling the given bounds (if one exists).
By requiring minimality, however, the problem becomes computationally hard.
In fact, it can be shown that the problem of whether there exists a DFA
with $n$ states for Problem \ref{problem:1} is \NP-complete. 

\begin{theorem}
\label{thm:NP-complete}
	Given a multi-set of words $\sample$, two natural numbers
	$\ell,u \in \mathbb{N}$ with $\ell \leq u \leq \abs{\sample}$,
	and a natural number $n$ (given in unary), the problem of finding
	a DFA $\dfa$ with $n$ states that accepts at least
	$\ell$ and at most $u$ words from $\sample$ is \NP-complete.
\end{theorem}

In order to prove Theorem~\ref{thm:NP-complete}, we begin by showing that the problem is in \NP, as formalized in the following Lemma.

    \begin{lemma}
    \label{thm:NP}
        Given a multi-set of words $\sample$, two natural numbers $\ell,u \in \mathbb{N}$ with $\ell \leq u \leq \abs{\sample}$, and a natural number $n$ (given in unary), a non-deterministic Turing machine can compute in polynomial time whether there exists a DFA $\dfa$ with $n$ states which accepts at least $\ell$ and at most $u$ words from $\sample$ (i.e., the problem lies in \NP).
    \end{lemma}

    \begin{proof} [of Lemma~\ref{thm:NP}]
    	Let $\sample,\ell,u,$ and $n$ be given.
        As $n$ is unary, a non-deterministic Turing machine can guess an automaton $\dfa$ with $n$ states.
        The number of accepted words from $\sample$ can then be computed in polynomial time by simulation.
        Finally, checking whether this number is at least $\ell$ and at most $u$ is also possible in polynomial time, showing that the problem lies in \NP. \qed
    \end{proof}

Next, we we show \NP-hardness of Problem~\ref{problem:1}, thus we conclude that it is \NP-complete.
    
    \begin{lemma}
    \label{thm:NP-hard}
    	Given a multi-set of words $\sample$, two natural numbers $\ell,u \in \mathbb{N}$ with $\ell \leq u \leq \abs{\sample}$, and a natural number $n$ (given in unary), the problem of finding a DFA $\dfa$ with $n$ states which accepts at least $\ell$ and at most $u$ words from $\sample$, is \NP-hard.
    \end{lemma}

    To proof Lemma~\ref{thm:NP-hard}, we use the \NP-completeness of the following problem (see \cite{DBLP:books/fm/GareyJ79}).

    \begin{problem}\label{problem:NP-Complete}
    	Given finite disjoint sets of words $P,N\subseteq\Sigma^*$ and a unary number $k\in\mathbb{N}_0$, does there exist a deterministic finite automaton $\dfa$ with $k$ states such that $\dfa$ accepts all words in $P$ and rejects all words in $N$.
    \end{problem}

    As detailed in the proof below, \NP-hardness (and thus \NP-completeness) of our problem follows by reduction from Problem \ref{problem:NP-Complete}.
    This reduction makes use of the multi-set structure of the samples to encode positive and negative words from Problem \ref{problem:NP-Complete} in different multiplicities in the multi-set, which can be distinguished by Problem \ref{problem:1}.

    \begin{proof} [of Lemma~\ref{thm:NP-hard}]
    	We show this by a many-one reduction from Problem \ref{problem:NP-Complete}.
    	Let finite disjoint sets $P,N\subseteq\Sigma^*$ and a unary number $k\in\mathbb{N}_0$ be given.
    	We construct an instance of our problem as follows:
        \begin{itemize}
            \item $n\coloneqq k$;
            \item $\sample(w)\coloneqq\begin{cases}
                                    \abs{N}+1&\text{if }w\in P,\\
                                    1&\text{if }w\in N,\\
                                    0&\text{otherwise,}
            \end{cases}$
    
    		for $w\in\Sigma^*$, i.e., the multi-set contains
    		exactly $(\abs{N}+1)$-times all words in $P$ and once all
    		words in $N$;
            \item $\ell=u\coloneqq\abs{P}(\abs{N}+1)$.
        \end{itemize}
        This construction can be done in polynomial time.
    
        Furthermore, if there exists a DFA $\dfa$ with $n$ states solving Problem \ref{problem:1} for $\sample,\ell,$ and $u$ as above, i.e., accepting exactly $\abs{P}(\abs{N}+1)$ words from $\sample$, we have that
        \[
        0\equiv\sum_{w\in P\cup N}\sample(w)\cdot\dfa(w)\equiv\sum_{w\in N}\dfa(w)\pmod{\abs{N}+1}
        \]
        Recall that $\sample(w)$ denotes the number of occurrences of $w$ in $\sample$ and $\dfa(w)$ indicates whether $w$ is accepted by $\dfa$.
    	This shows that $\dfa$ rejects all words in $N$, and, thus
        \[
        \abs{P}(\abs{N}+1)=\sum_{w\in P}\sample(w)\cdot\dfa(w)\Rightarrow\abs{P}=\sum_{w\in P}\dfa(w),
        \]
        showing that $\dfa$ accepts all words in $P$ whence $\dfa$ is also a solution to the instance $(P,N,k)$ of Problem \ref{problem:NP-Complete}.
    
        Similarly, if there exists a DFA $\dfa$ solving the instance $(P,N,k)$ of Problem \ref{problem:NP-Complete}, we have that
        \[
        \sum_{w\in\Sigma^*}\sample(w)\cdot\dfa(w)=(\abs{N}+1)\sum_{w\in P}\dfa(w)=\abs{P}(\abs{N}+1),
        \]
        whence $\dfa$ is also a solution with $n$ states to the instance $(\sample,\ell,u)$ of Problem \ref{problem:1}.
    
        All in all, this shows the reduction from Problem \ref{problem:NP-Complete}, and thus \NP-hardness of finding a solution of Problem \ref{problem:1} with $n$ states. \qed
    \end{proof}


In this first setup, we require the user to provide a lower and an upper bound on the distribution of positive words in a given sample.
However, in practice, this requirement may be too strong.
Thus, in the second setup, which we refer to as \emph{Single-Bound DFA Learning}, 
we reduce the amount of prior knowledge compared to the first case by removing requirement to specify both $\ell$ and $u$.
Instead, we assume to be given only one parameter, say $\ell$ (the case in which $u$ is given is analogous).
Now, in contrast to Problem~\ref{problem:1}, the task to construct a minimal DFA that accepts at least $\ell$ words from $\sample$ always has a trivial solution: the DFA that accepts all words in $S$ only has size 1 and fulfills the bound.
However, this DFA underfits the sample $\sample$ and thus does not capture the underlying structure.
To reduce underfitting, we apply a common technique from automata learning and construct a DFA of a fixed size that accepts the smallest number $k \geq l$ of words from the sample. 
By providing this size as an additional parameter $n$ the user can regularize the trade-off between avoiding underfitting (larger) and interpretability (smaller).
We formally state this problem as:
\begin{problem}[Single-Bound-Learning-Problem]
\label{problem:2}\\
Given a multi-set of words $\sample = \{w_1, \dots, w_n\}$ and two natural numbers $\ell, n \in \mathbb{N}$ with $\ell \leq \abs{\sample}$, construct a DFA $\dfa$ of size $n$ which accepts the smallest number $k \geq l$ of words from $\sample$.
\end{problem}
This problem is the optimization version of Problem~\ref{problem:1}, thus, it is also computationally hard.
In fact, from Problem~\ref{problem:1} being \NP-complete, it immediately follows that Problem~\ref{problem:2} lies within the complexity class \NPO (i.e., the class of optimization problems whose decision variant lies in \NP).

The third setup, which we refer to as \emph{Distance-Based DFA Learning}, is motivated by the assumption that for many applications, the pairs of positive (or negative) words are structurally similar, resulting in a low edit distance.
In contrast, opposite classifications (i.e, one positive and one negative word) are drastically different, resulting in a high edit distance.
In the context of anomaly detection, for instance, deep learning based methods follow this idea and classify new data based on the distance to the training data (e.g.,~\cite{pmlr-v80-ruff18a}).
This assumption allows us to alleviate the user’s burden to specify any bounds.
Instead, along with the sample $\sample$, we only rely on the size $n \in \mathbb N$ of the automata as a regularizer (similar to the second setting) and a distance function over words (in our case, the Levenshtein distance).
The task is then to construct a DFA of size $n$ such that the distance between all pairs of two accepted (rejected) words is minimized while the distance between pairs of both one accepted and one rejected word is maximized.
This dual optimization problem can then be transformed into a plain minimization problem by multiplying the distances to be maximized by $-1$.
This allows us to formally state this problem as:
\begin{problem}[Distance-Based-Learning-Problem]
\label{problem:3}\\
Given a multi-set of words $\sample = \{w_1, \dots, w_n\}$ and a natural number $n \in \mathbb{N}$, construct a DFA $\dfa$ of size $n$ which minimizes the following objective function:
\begin{align}
    \sum\limits_{w_i, w_j \in L(\dfa) \cap \sample} dist(w_i, w_j)  - \sum\limits_{\substack{w_i \in L(\dfa) \cap \sample \\ w_j \notin L(\dfa) \cap \sample}} dist(w_i, w_j) \label{eq:distance}
\end{align}
where $dist(w_i, w_j)$ denotes the Levenshtein distance between two words $w_i$ and $w_j$.
\end{problem}

While we conjecture that this problem is also computationally hard, we leave a proof of its complexity as part of future work.

%% file: learning.tex
\section{Learning via Discrete Optimization}
\label{sec:Learn}
In this section, we present our learning algorithms for learning deterministic finite automata from a sample $\sample$ of unlabeled words.
In all three setups, we reduce the tasks to a set of constraint optimization problems and solve them using state-of-the-art mixed-integer programming solvers (in our case, Gurobi~\cite{gurobi}).

\subsection{Two-Bound DFA Learning}
Recall that in the first setup, we are given a sample $\mathcal{S}$ and two bounds $\ell, u \in \mathbb{N}$ with $\ell \leq u \leq \abs{\sample}$.
To learn a minimal DFA that fulfills these bounds, we apply a technique commonly used in automata learning to ensure minimality. 
The idea is to encode the problem for an automaton of fixed size $n$ such that the encoding has two key properties:
\begin{itemize}
    \item There exists a feasible model if and only if there exists a DFA of size $n$ fulfilling the bounds on the acceptance.
    \item This model contains sufficient information to construct such a DFA.
\end{itemize}
Starting with an automaton of size one and increasing the size whenever there is no feasible model guarantees to produce the minimal solution.

We now describe the MILP model we use to learn a DFA of size $n$.
Since we just check the existence of a suitable DFA in the first setup, we can choose any constant objective function, e.g., $\mathit{obj} = 1$.
The set of linear inequalities $\Phi_{\sample, \ell, u}^n = \Phi_{\dfa}^n \land \Phi_\mathcal{B}$ consists of two kinds of constraints: 
\emph{automata constraints} $\Phi_{\dfa}^n$, which encode a DFA of size $n$ and the runs on all words from the sample, and \emph{bound constraints} $\Phi_\mathcal{B}$ encoding the bounds on the acceptance.
Throughout these constraints, we bound the introduced variables to only take on integer values between (and including) $0$ and $1$, thus simulating boolean variables.

\paragraph{Automata constraints $\Phi_{\dfa}^n$} 
The automata constraints are motivated by the SAT encoding of Biermann and Feldman~\cite{BiermannF72}.
Without loss of generality, the states of the DFA form the set
$Q = \{q_0, \dots, q_{n-1}\}$ where $q_0$ is the initial state.
The alphabet $\Sigma$ of the DFA is the set of all symbols appearing
in the sample $\sample$.
To encode the transitions of the DFA we introduce variables $\delta_{q,a,q'}$ for $q,q' \in Q$ and $a \in \Sigma$.
Intuitively, the variable $\delta_{q,a,q'}$ will be set to 1 if and only if the DFA has a transition from state $q$ to state $q'$ on reading $a$.
Furthermore, we introduce variables $f_q$ for $q \in Q$, which indicate whether a state $q$ is a final state.
To encode the runs of the DFA, we start by computing the set of all prefixes in the sample $\mathit{Pref}(\mathcal{S}) = \{ w \mid ww' \in \mathcal{S} \text{ and } w' \in \Sigma^*\}$.
We then introduce a third kind of variable:
$x_{w,q}$ for all $w \in \mathit{Pref}(\mathcal{S})$ and $q \in Q$.
Intuitively, these variables indicate that after reading the prefix $w$ the DFA is in state $q$.\\
We now impose constraints on these variables to encode a DFA and its runs. 
Being deterministic, there must be precisely one transition for each state $q$ and symbol $a$, which we can model by the following constraint:
\begin{align}
    \sum_{q' \in Q} \delta_{q,a,q'} = 1 \hspace{0.75cm} \forall q \in Q, \forall a \in \Sigma \label{eq:single-transition}
\end{align}
Furthermore, after reading a word $w$ the DFA can only be in one state and after reading the empty word $\epsilon$ the DFA is in the initial state, which we defined to be $q_0$.
\begin{align}
    \sum_{q \in Q} x_{w,q} = 1 \hspace{0.75cm} \forall w \in \mathit{Pref}(S) \hspace{0.75cm} \text{and} \hspace{0.75cm} x_{\epsilon,q_0} = 1 \label{eq:single-run}
\end{align}

Moreover, we encode a run based on the following observation:
If the DFA is in some state $q$ after reading the word $w$ and there is a transition from $q$ to $q'$ on reading the symbol $a$, then the DFA is in state $q'$ after reading the word $wa$.
As a constraint in MILP we get:
\begin{align}
    x_{w,q} + \delta_{q,a,q'} - 1 \leq x_{wa,q'} \hspace{0.75cm}  \forall q, q' \in Q, a \in \Sigma, \forall wa \in \mathit{Pref}(S) \label{eq:run}
\end{align}
Finally, we define the automata constraints $\Phi_{\dfa}^n$ to be the conjunction of Equations~\ref{eq:single-transition} to~\ref{eq:run} which concludes the encoding of a DFA of size $n$ and its runs.

\paragraph{Bound constraints $\Phi_\mathcal{B}$}:
To impose constraints on the number of accepted words, we need to track whether a word $w$ is accepted.
This is the case if and only if after reading $w$ the DFA is in some state $q$ and this state is final.
Intuitively, we could express this case as $x_{w,q} \cdot f_q$, however, this is not linear and thus not a valid MILP constraint.
Instead, we exploit the fact that for Boolean variables the multiplication $x_{w,q} \cdot f_q$ is equivalent to the formula $x_{w,q} \land f_q$. 
By introducing fresh variables $\alpha_{w,q}$ for $w \in \sample$ and $q \in Q$ to store the result, this conjunction can be modeled by the following set of constraints:
\begin{align}
    \alpha_{w,q} \geq x_{w,q} + f_q - 1, \quad
    \alpha_{w,q} \leq x_{w,q}, \quad 
    \alpha_{w,q} \leq f_q \hspace{0.75cm}\forall w \in \sample, \forall q \in Q \label{eq:accepting}
\end{align}
Intuitively, the variables $\alpha_{w,q}$ indicate whether a word $w$ is accepted by the DFA (in the state $q$).
Relying on these variables, we can encode the bounds on the acceptance as 
\begin{align}
    \sum_{w \in S} \sum_{q \in Q} \alpha_{w,q} \geq \ell
    \hspace{0.75cm} \text{and} \hspace{0.75cm}
    \sum_{w \in S} \sum_{q \in Q} \alpha_{w,q} \leq u 
    \label{eq: bounds}
\end{align}
Then the bound constraints $\Phi_\mathcal{B}$ are defined as the conjunction of the above inequalities.

After introducing the MILP model, we employ it to construct the minimal DFA that fulfills the given bounds on the acceptance.
The idea is to check the feasibility of the MILP problem with constant objective function $\mathit{obj} = 1$ and linear inequalities $\Phi_{\sample, \ell, u}^n$ for increasing $n$ until either a solution is found or we reach $n = \abs{\mathit{Pref}(\mathcal{S})} + 2$.
As argued above when proving decidability of Problem~\ref{problem:1}, the size of the prefix tree is a natural upper bound for the size of the DFA.
Therefore, we can conclude that there exists no DFA fulfilling the given bounds on the sample when we exceed this size.
In the case where a feasible model exists for some size $n$, we construct the corresponding DFA from this model based on the variables $\delta_{q,a,q'}$ and $f_q$.
This procedure is described by Algorithm~\ref{alg:two-bounds}.
\begin{algorithm}[tbp]
	\caption{Learning with two bounds}\label{alg:two-bounds}
	
	\begin{algorithmic}[1]
		\State \textbf{Input:} Sample $\sample$, Bounds $\ell, u \in \mathbb{N}$
		\State $n\gets 0$
		\Repeat
		\State $n\gets n+1$
		\State Construct $\Phi_{\sample, \ell, u}^n = \Phi_{\dfa}^n \land \Phi_\mathcal{B}$ and set $\mathit{obj} = 1$
		\If{$\mathit{obj},\Phi_{\sample, \ell, u}^n$ has a feasible model (say $m$)}
			\State \Return Construct DFA $\dfa$ using $m$
		\EndIf
		\Until{$n = \abs{\mathit{Pref}(\mathcal{S})} + 2$}
		\State \Return There exists no DFA fulfilling the given bounds
	\end{algorithmic}
\end{algorithm}
The correctness of this algorithm is established by the following theorem:
\begin{theorem}
\label{thm:correctness-two-bounds}
Given a sample $\sample$ and two natural numbers $\ell, u \in \mathbb{N}$ with $\ell \leq u \leq \abs{\sample}$, Algorithm~\ref{alg:two-bounds} terminates and outputs a minimal DFA $\dfa_{\sample}$ which accepts at least $\ell$ and at most $u$ words from $\sample$, if such a DFA exists.
\end{theorem}

\begin{proof} [of Theorem~\ref{thm:correctness-two-bounds}]	
	We prove Theorem~\ref{thm:correctness-two-bounds} in three steps:
	First, we explain how we construct a DFA from a feasible
	model and proof that this automaton is well-defined and solves
	Problem~\ref{problem:1}.
	Afterwards, we show that a feasible model exists for a size $n$
	if and only if there exists a DFA of that size fulfilling the bounds.
	In the end, we establish termination and show that Algorithm~\ref{alg:two-bounds}
	finds a DFA fulfilling the bounds if such a DFA exists.
	By construction this DFA is minimal.
	
	For now let us assume we found a feasible model for some size $n$.
	We show that the DFA $\dfa=(Q,\Sigma,q_I,\delta,F)$ given by
	\begin{itemize}
		\item $Q=\{q_0,\dots,q_{n-1}\},q_I=q_0;$
		\item $\Sigma$ the symbols present in $\sample$;
		\item $\delta:Q\times\Sigma\rightarrow Q,(q,a)\mapsto q'$
			for $q\in Q,a\in\Sigma,$ and $q'\in Q$ such that
			$\delta_{q,a,q'}$ = 1;
		\item $F\coloneqq\{q\in Q\mid f_q = 1\};$
	\end{itemize}
	is well-defined and solves Problem~\ref{problem:1}.
	First of all, Constraint \ref{eq:single-transition} ensures that the state-transition function $\delta$ is well-defined while $f_q$ simulating Boolean variables further ensures that $F$ is well-defined.
	Next, we show that the variables $x_{w,q}$ correspond to the runs of words $w$ from $\sample$ in this well-defined DFA.
	More precisely, we show that $x_{w,q}=1$ if $\dfa$ is in state $q\in Q$ after reading $w\in\mathit{Pref}(\mathcal{S})$ by induction over the prefix length $k=|w|$ (where we w.l.o.g. assume $\sample$ to be non-empty).
	Note that Constraint \ref{eq:single-run} then also ensures the opposite
	direction, i.e., that $\dfa$ is in state $q\in Q$ after reading
	$w\in\mathit{Pref}(\mathcal{S})$ if $x_{w,q}=1$ (as $x_{w,\tilde{q}}=1$
	for the true state $\tilde{q}$ and thus $x_{w,q'}=0\neq1$ for all other
	states $q'\neq\tilde{q}$).
	
	\textit{Base case.} 
	The only prefix of length $k=0$ obviously being the empty word
	$\varepsilon$, Constraint \ref{eq:single-run} gives that
	$x_{\varepsilon,q_0}=1$, showing the statement for prefixes of
	length $k=0$.

	\textit{Induction step.}
	Assuming that the statement holds for any prefix of length less or equal to $k$. 
	If no prefix of length $k+1$ exists in $\sample$, the induction is closed and the statement is proven for all prefixes.
	Assume now that $wa\in\mathit{Pref}(\sample)$ is a prefix of length $k+1$.
	Then $w$ is a prefix of length $k$ and thus fulfills for the state
	$q$ reached after reading $w$ in $\dfa$ that $x_{w,q}=1$.
	Writing $q'\coloneqq\delta(q,a)$, Constraint \ref{eq:run} and the
	definition of $\delta$ then give \[x_{w,q} + \delta_{q,a,q'} - 1 = 1 \leq x_{wa,q'}\]
	whence also $x_{wa,q'}=1$ as a Boolean variable.
	As $wa$ was an arbitrary prefix of length $k+1$, this concludes the induction.
	
	While this proofs that the DFA $\dfa$ is well-defined, the bound constraints $\Phi_\mathcal{B}$ ensure that the number of accepted words is above the lower bound $\ell$ and below the upper bound $u$.
	Hence, the DFA $\dfa$ is well-defined and solves Problem~\ref{problem:1}
	
	In a second step, we show that our MILP problem has a feasible
	model for size $n$ if and only if there exists a DFA of size $n$
	that accepts at least $\ell$ and at most $u$ words from $\sample$.

	$(\Rightarrow):$ Given a feasible model for size $n$, we construct a DFA $\dfa$ as explained above.
	As displayed this DFA $\dfa$ is well-defined and fulfills the bounds on the acceptance.

	$(\Leftarrow):$ Given a DFA $\dfa$ of size $n$ which accepts at
	least $\ell$ and at most $u$ words from $\sample$, let the model
	be given by the natural interpretation of the variables:
	\begin{align*} 
		\delta_{q,a,q'} & \coloneqq\begin{cases}
			1&\text{if }\delta(q,a)=q',\\
		    0&\text{otherwise,}
		\end{cases}\\
		f_q & \coloneqq\begin{cases}
			1&\text{if }q\in F,\\
		    0&\text{otherwise,}
		\end{cases}\\
		x_{w,q} & \coloneqq\begin{cases}
			1&\text{if }\dfa\text{ is in state }q\text{ after reading }w,\\
		    0&\text{otherwise,}
		\end{cases} \\
		\alpha_{w,q} & \coloneqq\begin{cases}
			1&\text{if }\dfa\text{ is in state }q\text{ after reading }w\text{ and }q\in F,\\
			0&\text{otherwise,}
		\end{cases}
	\end{align*}
	Being deterministic the DFA $\dfa$ and thus any model constructed this way obviously fulfills Constraints~\ref{eq:single-transition} to~\ref{eq:run}.
	Furthermore, the definition of $\alpha_{w,q}$ ensures that the set of Constraints~\ref{eq:accepting} is fulfilled and that $\sum_{w\in\sample}\sum_{q\in Q}\alpha_{w,q}=\sum_{w\in\sample}\sum_{q\in F}x_{w,q}$ corresponds to the amount of words in $\sample$ which are accepted by $\dfa$ whence by assumption both Constraints \ref{eq: bounds}  are fulfilled.
	Hence this model is feasible for our MILP problem for size $n$.

	In order to conclude the proof of Theorem~\ref{thm:correctness-two-bounds}, we now show that Algorithm~\ref{alg:two-bounds}
	terminates and finds a DFA fulfilling the bounds, if such a DFA exists.
	Termination itself is straight forward.
	The algorithm iterates over increasing sizes until a feasible model is found or $n = \abs{\mathit{Pref}(\mathcal{S})} + 2$ is reached in which case it concludes that no DFA fulfilling the bounds exists.
	Since solving the MILP problem for each $n$ is computable, the algorithm thus always terminates.
	Furthermore, if a feasible model is found, we showed above that the algorithm returns a DFA fulfilling the bounds.
	On the other hand, as explained in the decidability proof of Problem~\ref{problem:1},
	the size of the prefix tree is a natural upper bound for the minimal DFA fulfilling the bounds.
	In the prefix tree, the run on each word $w$ in the sample $\sample$ leads to a unique state $q_w$.
	Therefore, we can check accepting each combination of words from $\sample$ by making the corresponding states a final or non-final state respectively.
	If no such combination fulfills the bounds, we can conclude that no DFA exists which fulfills the bounds.
	Since the prefix tree can have unspecified transitions, we may need one additional state to which we target all these unspecified transitions in order to construct a well-defined DFA.
	Hence, there either exists a DFA of size $n = \abs{\mathit{Pref}(\mathcal{S})} + 1$ or none at all.
	By the equivalence proven above, we thus have that not finding a feasible model until $n = \abs{\mathit{Pref}(\mathcal{S})} + 2$ shows that truly no DFA fulfilling the bounds exists.
	This concludes the proof of termination and correctness and, thus, of Theorem~\ref{thm:correctness-two-bounds}.

	We have shown that Algorithm~\ref{alg:two-bounds} terminates and finds a DFA which fulfills the bounds on the acceptance, if such a DFA exists.
	We displayed that such a DFA of size $n$ exists if and only if our MILP problem for size $n$ has a feasible model.
	Furthermore, we have shown how to construct this DFA given a feasible model. \qed
\end{proof}

Finally, let us investigate the number of constraints in our MILP model.
This number depends on multiple factors:
\begin{itemize}
	\item The size $n$ of the automaton to be constructed
	\item The size of the alphabet $\Sigma$ over which the words in the sample $\sample$ range
	\item The number of unique words in $\sample$
	\item The size of the prefix tree of $\sample$
\end{itemize}
Let $\abs{\Sigma}$ now denote the size of the alphabet and $p =\abs{\mathit{Pref}(\mathcal{S})}$ denote the size of the prefix tree of $\sample$.
Note that the number of unique words in $\sample$ is a lower bound for the size of the prefix tree.
Then, we obtain the following remark.

\begin{remark}
The number of constraints in the MILP problem is in $\mathcal{O}(n^2 \cdot \abs{\Sigma} \cdot p)$.
\end{remark}

\subsection{Single-Bound DFA Learning}
In order to not clutter this section, we will only describe the encoding for the case where the lower bound $\ell$ is given. 
The case in which the upper bound is given is analogous.
We recall that the task is to construct a DFA of a fixed size $n$ that minimizes acceptance above $\ell$.
As in the first setup, we encode the DFA and the runs on all words in the sample using the same set of variables and automata constraints $\Phi_{\dfa}^n$ as above.
Furthermore, we use the same idea to ensure that the number of accepted words is larger than the lower bound:
We introduce variables $\alpha_{w,q}$ and add Constraints~\ref{eq:accepting} and the corresponding inequality of Constraint~\ref{eq: bounds}.
For the remainder of this section, let $\Phi_{\ell}$ denote the conjunction of these constraints.
In contrast to Two-Bound Learning, we are not satisfied with finding just any DFA, but want to find one that accepts the least number of words while adhering to the lower bound.
To achieve this, we use the following objective function:
\begin{equation}\label{eq:optimization}
    \mathit{obj} = \min\ \sum_{w \in S} \sum_{q \in Q} \alpha_{w,q}
\end{equation}
which minimizes the number of accepted words from $\sample$.
All in all, the resulting Algorithm~\ref{alg:lower-bound} returns a DFA of size $n$ that minimizes acceptance above the lower bound $l$.
\begin{algorithm}[b]
	\caption{Learning with a single bound}\label{alg:lower-bound}
	
	\begin{algorithmic}[1]
		\State \textbf{Input:} Sample $\sample$, Bound $\ell \in \mathbb{N}$, Size $n \in \mathbb{N}$
		\State Construct $\Phi_{\sample, \ell}^n = \Phi_{\dfa}^n \land \Phi_{\ell}$
		\State Set $\mathit{obj} = \min\ \sum_{w \in S} \sum_{q \in Q} \alpha_{w,q}$
		\State Compute optimal model minimizing $\mathit{obj}$ with respect to $ \Phi_{\sample, \ell}^n$, say $m$
		
		\State \Return Construct DFA $\dfa$ using $m$
	\end{algorithmic}
\end{algorithm}
\begin{theorem}
Given a sample $\sample$ and two natural numbers $\ell, n \in \mathbb{N}$ with $\ell \leq \abs{\sample}$, Algorithm~\ref{alg:lower-bound} terminates and outputs a DFA $\dfa_{\sample}$ of size $n$ that accepts the smallest number $k \geq l$ of words from $\sample$.
\end{theorem}
We omit the proof of this theorem, which is similar to the proof of Theorem~\ref{thm:correctness-two-bounds}.

\subsection{Distance-Based DFA Learning}
Let us recall the third setup: We are given a sample $\sample$, the size $n \in \mathbb N$ of the automata, and the Leveshtein distance betwwen samples and want to compute a DFA that minimizes Equation~\ref{eq:distance}.
Analogously to the first two setups, we encode the DFA and the runs on all words in the sample using the same set of variables and automata constraints $\Phi_{\dfa}^n$ as above.
To optimize the distance between pairs of words as described in Problem~\ref{problem:3}, we need to keep track of which words will be accepted by the automaton and which will be rejected.
For this, we introduce variables $\alpha_{w,q}$ and $\beta_{w,q}$, respectively, and add Constraints~\ref{eq:accepting} to track whether a word $w$ is accepted by the DFA as well as $\beta_{w,q} = 1 - \alpha_{w,q}$ for all $w \in \sample$ and $q \in Q$, indicating whether a word $w$ is rejected by the automaton.
We then define the objective function in the same way as in Problem~\ref{problem:3}:
\begin{equation}
    \mathit{obj} = \min\ \sum_{w_1, w_2 \in S} \sum_{q_1,q_2 \in Q} \alpha_{w_1,q_1} \cdot \alpha_{w_2,q_2} \cdot dist(w_1, w_2) -  \alpha_{w_1,q_1} \cdot \beta_{w_2,q_2} \cdot dist(w_1, w_2)
\end{equation}
Note that for every pair of words $w_1$ and $w_2$, the distance $dist(w_1, w_2)$ can be precomputed, thus allowing arbitrary complex distance functions to be used.
As above, this MILP model can then be used to construct a DFA of a given size $n$ that minimizes the distance between pairs of accepted words while maximizing the distance between pairs of one accepted and one rejected word (see also Algorithm~\ref{alg:distance}).
\begin{algorithm}[tbp]
	\caption{Learning based on distance}\label{alg:distance}
	
	\begin{algorithmic}[1]
		\State \textbf{Input:} Sample $\sample$, Size $n \in \mathbb{N}$
        \State Compute the Levenshtein distance $dist(w_1, w_2)$ for each pair of words $w_1,w_2 \in \sample$
		\State Construct $\Phi_{\dfa}^n$
		\State Set $\mathit{obj} = \min\ \sum\limits_{\substack{w_1, w_2 \in S \\ q_1,q_2 \in Q}} \alpha_{w_1,q_1} \cdot \alpha_{w_2,q_2} \cdot dist(w_1, w_2) -  \alpha_{w_1,q_1} \cdot \beta_{w_2,q_2} \cdot dist(w_1, w_2)$
		\State Compute optimal model minimizing $\mathit{obj}$ with respect to $\Phi_{\dfa}^n$, say $m$
		
		\State \Return Construct DFA $\dfa$ using $m$
	\end{algorithmic}
\end{algorithm}
\begin{theorem}
Given a sample $\sample$ and two natural numbers $\ell, n \in \mathbb{N}$ with $\ell \leq \abs{\sample}$, Algorithm~\ref{alg:distance} terminates and outputs a DFA $\dfa_{\sample}$ of size $n$ which minimizes 
the following objective function: $\sum\limits_{w_i, w_j \in L(\dfa)} dist(w_i, w_j)  - \sum\limits_{\substack{w_i \in L(\dfa) \\ w_j \notin L(\dfa)}} dist(w_i, w_j)$

\end{theorem}
We omit the proof of this theorem which is similar to the proof of Theorem~\ref{thm:correctness-two-bounds}.

%% file: interpretability.tex
\section{Interpretability}
\label{sec:interpre}

Even though automata are generally regarded as interpretable models \cite{Shvo2021InterpretableSequenceClassification},
they can become unintuitive if there are too many different transitions
to different states.
Therefore, we introduce heuristics aimed at reducing their complexity
and making them more readable for humans.
While \cite{Shvo2021InterpretableSequenceClassification} introduced similar techniques as regularization terms, we adapt them to improve the interpretability of the resulting models, as demonstrated in Figure~\ref{fig:interpretability}.
Note though, that the model's obtained by different simplification heuristics need not be equivalent and that thus the modifications may impede or even improve the models accuracy (see Section~\ref{sec:eval}).
In essence, the following heuristics aim to visually streamline the graphical
representation of the resulting model and thus highlight the important structural insights.
They are implemented by adding a penalty term to the objective function.

\begin{itemize}
    \item \textit{Sink states}:  We favor solutions that have a so-called sink state, which can never be left once it is reached. By our design, all words ending in the sink state are rejected.
        
		To introduce a sink state $q_1$ (i.e., a non-final state with only self-loops) to the automaton, we add the following constraints:
		\begin{align*}
		    \delta_{q_1, a, q_1} & = 1 & \forall a \in \Sigma \\
		    f_{q_1} & = 0 
		\end{align*}
		    Moreover, we add
		\[
			\lambda_s \cdot \left( \sum_{q \in Q, a \in \Sigma} 1 - \delta_{q, a, q_1} \right)
		\]
		to the objective function, which penalizes each transition not targeting the sink state.
		The parameter $\lambda_l \in \mathbb{R}$ is a weight term to be chosen by the user. 
		Note that we need to have at least two states in our DFA to have a sink-state. \\[1ex]
    \item \textit{Self-loops}: By penalizing transitions to other states, we obtain models with a lot of self-loops. By convention, those are omitted in the graphical representation.

		To increase the number of self-loops for the resulting automata, we add
		\[
			\lambda_l \cdot \left( \sum_{q \in Q, a \in \Sigma} \sum_{q' \in Q \setminus \{q\}} \delta_{q, a, q'} \right)
		\]
		to the objective function. This term penalizes each transition
		where the source state is different from the destination state (i.e., not a self-loop).
		Here, $\lambda_l \in \mathbb{R}$ is a weight term to be chosen by the user. \\[1ex]
    \item \textit{Parallel edges}: Similar to self-loops, we prefer solutions where there is only one successor state. Thus, the automata will transition to the same state regardless of the next element $a \in \Sigma$.
        
		Similar to self-loops, we can increase the number of parallel edges by adding
		\[
			\lambda_p \cdot \left( \sum_{q \in Q} \sum_{q' \in Q} e_{q, q'} \right)
		\]
		to the objective function. 
		The boolean variable $e_{q, q'}$ is equal to $1$ if and only if there is at least one transition from $q$ to $q'$
		and can be computed using the following set of constraints:
		\begin{align*}
		    e_{q, q'} & \leq \sum_{a \in \Sigma} \delta_{q, a, q'} & \forall q, q' \in Q \\
		    e_{q, q'} & \geq \delta_{q, a, q'} & \forall q, q' \in Q, \forall a \in \Sigma
		\end{align*}
		These constraints simply model the boolean function $e_{q, q'} \leftrightarrow \bigvee_{a \in \Sigma} \delta_{q, a, q'}$.
		Again, $\lambda_p \in \mathbb{R}$ is a weight term to be chosen by the user.

\end{itemize}

Note that our optimization-based approach is flexible with respect to the heuristics used: as long as a heuristic is expressible as part of the optimization problem, it can be applied. 
Furthermore, this approach is orthogonal to the original encoding and can, in principle, also be applied to other constraint-based learning algorithms for finite-state machines.

\begin{figure}[tbp]
    \centering
    \begin{subfigure}[t]{0.25\textwidth}
        \raggedright
        \includegraphics[scale=0.40]{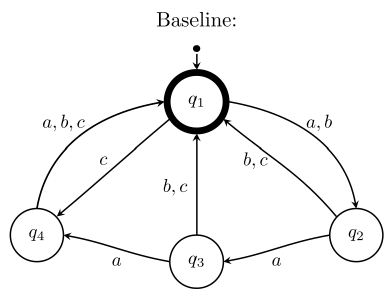}
    \end{subfigure}%
    ~ 
    \begin{subfigure}[t]{0.25\textwidth}
        \raggedright
        \includegraphics[scale=0.40]{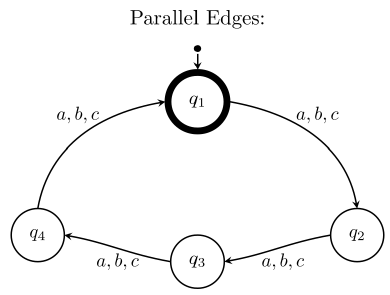}
    \end{subfigure}%
    ~ 
    \begin{subfigure}[t]{0.25\textwidth}
        \raggedright
        \includegraphics[scale=0.40]{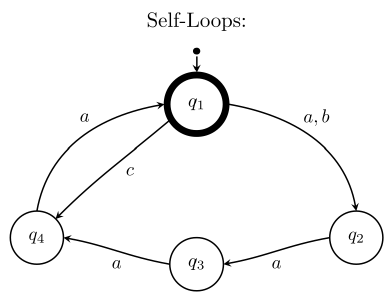}
    \end{subfigure}%
    ~ 
    \begin{subfigure}[t]{0.25\textwidth}
        \raggedright
        \includegraphics[scale=0.40]{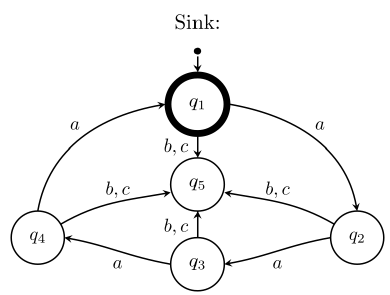}
    \end{subfigure}%
    \caption{Depiction of four different DFAs learned with different
    interpretability heuristics based on the same input dataset. 
    }
    \label{fig:interpretability}
\end{figure}

%% file: evaluation.tex
\section{Experimental Evaluation}
\label{sec:eval}

We implemented a prototype of the three learning algorithms in Python\footnote{\url{https://github.com/simonlutz-tudortmund/Interpretable-Anomaly-Detection}} using the industry-strength Gurobi Optimizer~\cite{gurobi} as an MILP solver.

We evaluated all three learning settings in the context of anomaly detection on three datasets: a modified version of the  ALFRED benchmark set~\cite{Shridhar2020AlfredBenchmarkInterpreting} and the two real-world log datasets HDFS~\cite{xu2009hfds} and BGL~\cite{oliner2007bgl}, provided by the Loghub system log dataset collection~\cite{He2020LoghubLargeCollection}. 
The exact dataset characteristics are shown in Table~\ref{tbl:tb-datasets}.

The ALFRED data set contains sequences of action plans, encoded as bit vectors, that achieve one of 7 goals in the ALFRED setting.
For each of the pairwise combinations of goals (i.e., 42 class combinations), we created a training and a test set with the elements of the first (normal) class and the elements of the second (anomalous) class in a 9:1 ratio. 

The HDFS data set contains system logs for a Hadoop Distributed File System hosted in a private cloud environment.
Each entry in the data set represents a sequence of system events 
, labeled as either normal or anomalous by a set of expert rules. 
Similarly, BGL is a set of logs collected from a BlueGene/L supercomputer system, containing alert and non-alert messages. 
To keep the running time within the timeout of two hours, we restricted ourselves to words with a maximum length of 15 and 10, respectively.

\begin{table}[tbp]
  \centering
  \setlength{\tabcolsep}{12pt}
  \begin{tabular}{lrrrrrr}
    \toprule
    \multicolumn{1}{c}{Dataset} & \multicolumn{1}{c}{$|S|$} & \multicolumn{1}{c}{$|\Sigma|$} & \multicolumn{1}{c}{\% anomalies}  & \multicolumn{1}{c}{LB} & \multicolumn{1}{c}{UB}\\
    \midrule
    ALFRED & 316-462 &  9 & $\approx$ 0.1 & 0.09-0.10 & 0.10-0.11 \\
    HDFS & 108237 &  13 & 0.0585 & 0.058 & 0.059 \\
    BGL & 198192 & 295 & 0.1766 & 0.176 & 0.177 \\    
    \bottomrule
  \end{tabular}
  \caption{Summary of the datasets used.}
  \label{tbl:tb-datasets}
\end{table}

\paragraph{Performance analysis.}

\begin{figure}[tbp]
  \centering
  \begin{subfigure}{0.495\textwidth}
    \centering
    \includegraphics[width=\textwidth]{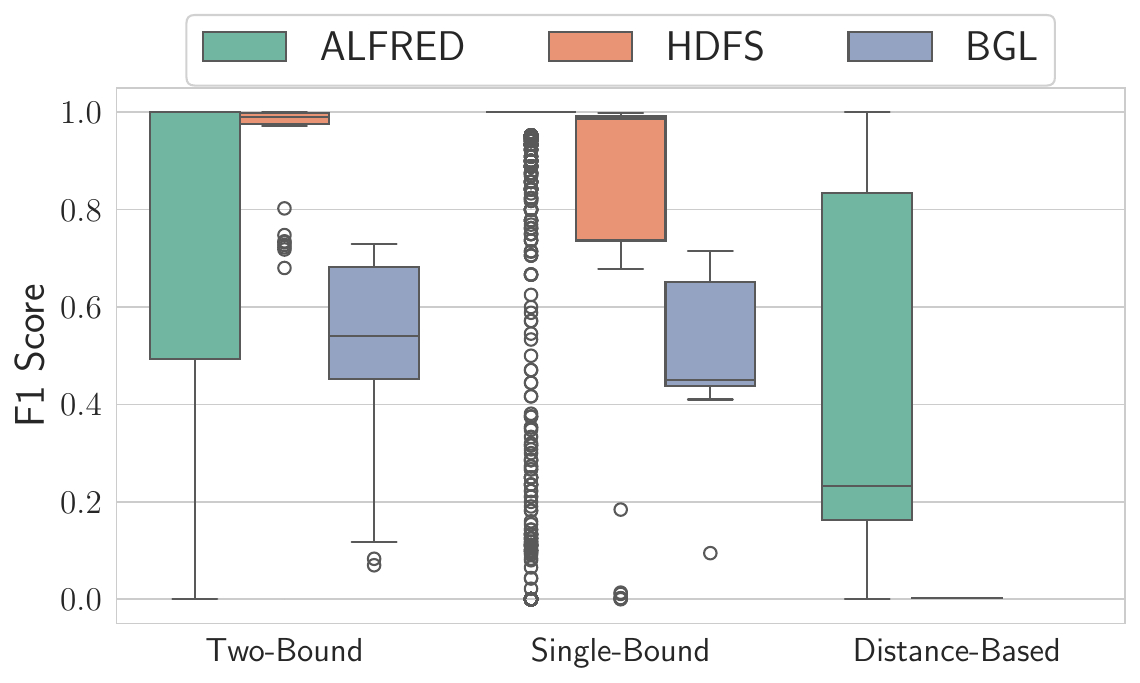}
    \caption{F1 score achieved.}
    \label{fig:f1_score}
  \end{subfigure}
  \hfill
  \begin{subfigure}{0.495\textwidth}
    \centering
    \includegraphics[width=\textwidth]{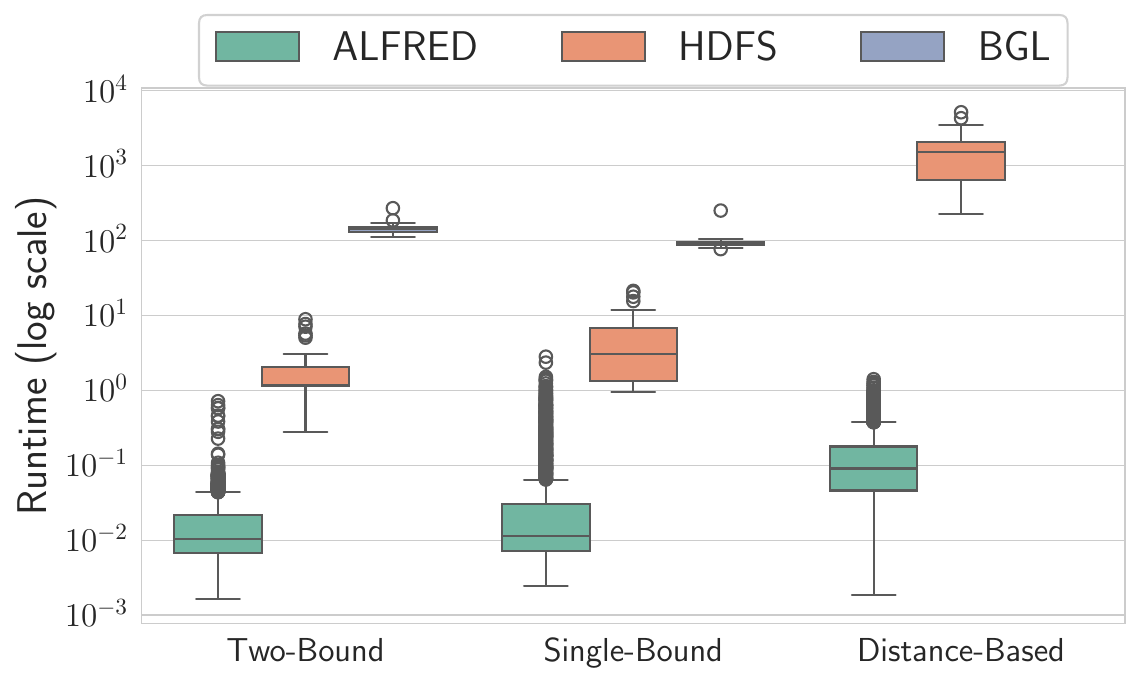}
    \caption{Time required to solve the model.}
    \label{fig:runtime}
  \end{subfigure}
  \caption{Results by three approaches on the selected datasets.}
  \label{fig:results}
\end{figure}

In our first experiment, we want to answer the research question asking which learning setting performs the best for detecting anomalies in sequential data.
For each of the three datasets, we randomly split the data into training and test set (with a 80:20 ratio) and averaged our results over 50 runs.
We examined both the time required to build and solve the model during training (with a two hour timeout) and the F1 score obtained on the test set.
We observed that in every experiment the Two-Bound DFA Learning algorithm was able to find a feasible solution of size two. 
Therefore, for the other two learning settings, we also learned an automaton of size 2.
The results of our experiments are displayed in Figure~\ref{fig:results}. 
They show that both the Two-Bound and Single-Bound learning setting were able to achieve high performance on the ALFRED and HDFS datasets.
The performance on the BGL dataset is slightly worse, which can be explained by the significantly larger alphabet compared to the length of the words. 
This means that simple patterns, such as a single letter, are sufficient to differentiate even two normal sequences.
These patterns may then be picked up by our algorithms instead of patterns separating normal sequences and anomalies.
In terms of running time there seems to be no significant difference between the first two learning settings.
However, the results show that the running time of the algorithms heavily depends on the dataset. 
This is to be expected, as the number of automata constraints $\Phi_{\dfa}^n$ scales with the number of words (and prefixes) in the sample $\sample$, yielding a longer running time.
Compared to the first two settings, the Distance-Based learning algorithm performs significantly worse, especially for the HDFS dataset.
This can be explained by the fact that the HDFS dataset contains words of different length (ranging from 2 to 15), thus the distance between two words is dominated by their difference in length.
The algorithm, by minimizing Equation~\ref{eq:distance}, will therefore separate words mostly based on their length, neglecting any other differences or patterns.


\paragraph{Loosened Bounds.}
For the Two-Bound and Single-Bound learning settings we assume to be given bounds on number of words to be accepted. 
In this second experiment, we will answer the research questions whether these learning algorithms are robust under imprecise bounds.
Focusing on the ALFRED dataset, we analyze the effect of loosening the learning bounds on the F1 score by increasing or decreasing the bounds by a value between 0 and 0.05.
The results are displayed in Figure~\ref{fig:loose-results}.
They show, that the F1 score is highest for the tightest bounds, as expected.
Furthermore, the performance of both the Two-Bound and Single-Bound learning algorithm is quite robust to loosening the bounds.
The F1 score drops only slightly even for a loosened bound of 0.05, which represents a 50\% deviation from the original number of anomalies.

\begin{figure}[tbp]
  \centering
  \includegraphics[scale=0.6]{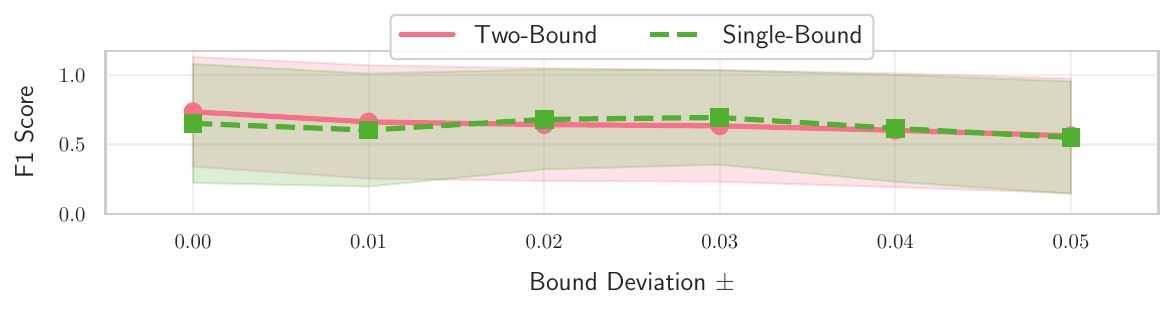}
  \caption{F1 score for the ALFRED dataset with loosened bounds.}
  \label{fig:loose-results}
\end{figure}


\paragraph{Data Complexity.}
For all datasets, the Two-Bound setting always found a feasible solution of size two.
Since this seemed rather small, we also trained a DFA in the classical passive learning setting (i.e., with labels) on the same datasets.
The resulting automata had size four for the ALFRED dataset and size two for the HDFS dataset, while no DFA with up to thirty states could be found for BGL.
This shows that even small automata are capable of separating the normal sequences and anomalies.

\begin{figure}[tbp]
  \centering
  \includegraphics[height= 0.15\textheight, width= \textwidth]{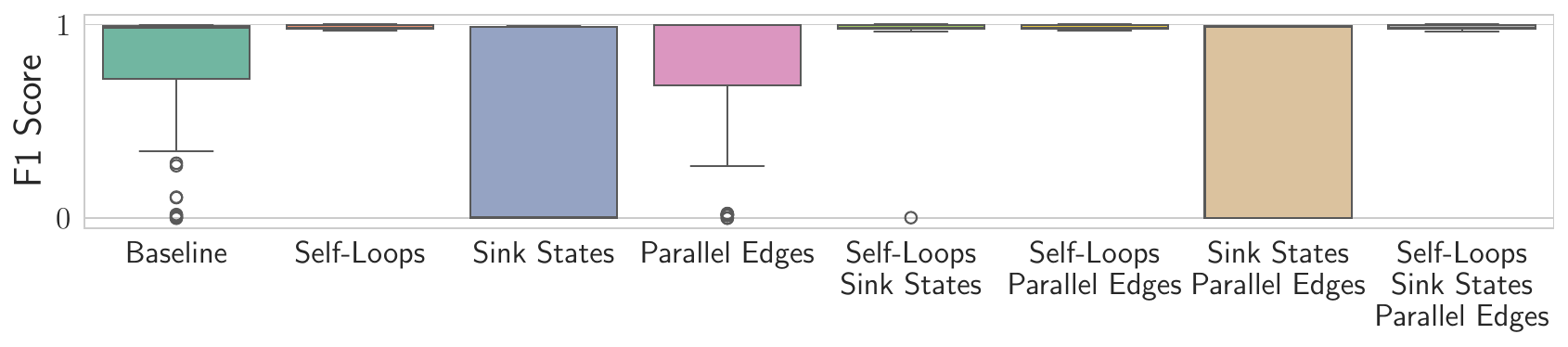}
  \caption{F1 score for combinations of interpretability heuristics on the HDFS dataset.}
  \label{fig:f1_interpretability}
\end{figure}

\paragraph{Influence of Interpretability Heuristics.}
In this section we investigate the influence of our interpretability heuristics on the algorithm's overall performance.
For each possible combination of heuristics, we learned an automaton in the first learning setting on the HDFS dataset with tight bounds.
Since any reasonable automaton that includes a sink state has a minimum of three states, we set the minimum number of states of the learned DFAs to three.
The results, shown in Figure~\ref{fig:f1_interpretability}, indicate that increasing the number of self-loops improves the F1 score.
Including more parallel edges does not seem to affect the overall performance, whereas  introducing sink states greatly impedes the resulting F1 score.
When combining multiple heuristics, the overall performance behaves the same as in one of the experiments using only a single heuristic.
This indicates that the influence of one heuristic on the objective function outweighs the other heuristics.
We leave a more thorough evaluation of the interpretability heuristics as part of future work.


\begin{figure}[tbp]
  \centering
  \includegraphics[scale=0.6]{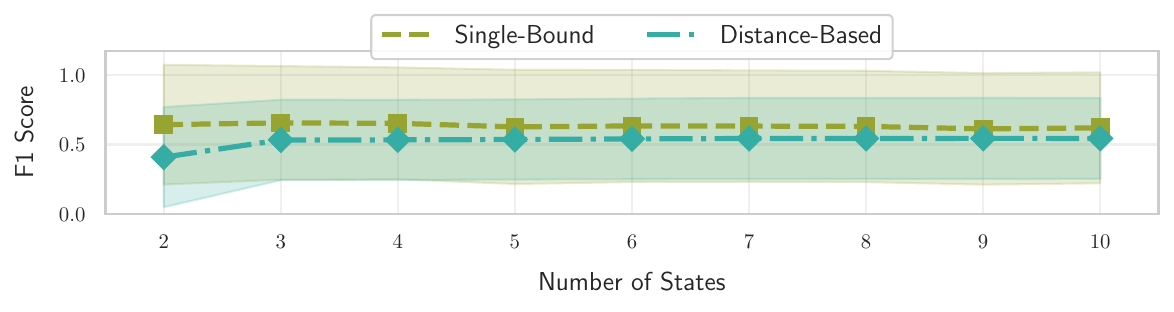}
  \caption{F1 score for the ALFRED dataset with changing number of states.}
  \label{fig:f1_states}
\end{figure}

\paragraph{Influence of Automaton Size.}
In the second and third learning setting, we assume the size of the automaton to be given by the user.
In this section, we investigate how this influences the overall performance of the learned automata. 
We conducted experiments with varying sizes (ranging from $2$ to $10$) on the ALFRED dataset.
The results are displayed in Figure~\ref{fig:f1_states}.
They indicate that increasing the size of the learned automaton does not significantly impact its performance.
For the Single-Bound setting the performance remains roughly constant, while for the  Distance-Based approach there is only a slight increase in performance between size $2$ and $3$.
This indicates that once a reasonable solution is found the learning algorithm starts to only learn unreachable states when the size is increased.
These results may be connected to the complexity of our data and we leave a more thorough investigation as part of future work.



%% file: conclusion.tex

\section{Conclusion}
This paper has studied the task of learning a deterministic finite automaton from a sample of unlabeled words that could be used to separate normal form anomalous words.
We proposed three unsupervised learning settings, studied their properties (e.g., their computational complexity), and developed  learning algorithms that utilize off-the-shelf constraint optimization tools.
In addition, we have shown how regularization can improve the interpretability of the learned DFAs.
Our empirical evaluation has demonstrated practical feasibility in the context of three anomaly detection benchmarks.

We see various promising directions for future research.
First, the analysis of the complexity of the third learning setting remains an open problem to be tackled in the future. 
Second, we plan to develop heuristics that sacrifice the optimality of a solution in favor of computational efficiency.
Third, we want to extend our approach to more expressive automata classes, for instance, register automata, to handle data over continuous domains.